\documentclass[]{article}
\usepackage{proceed2e}

\usepackage{times}
\usepackage{graphicx} 
\usepackage{subfigure} 

\usepackage[round]{natbib}

\usepackage{algorithm}
\usepackage{algorithmic}

\usepackage{hyperref}

\usepackage{graphicx}
\usepackage[T1]{fontenc}
\usepackage{graphics}
\usepackage{amsmath, amssymb,amsthm}
\usepackage{bm}
\usepackage{tikz}
\usetikzlibrary{positioning}
\usepackage{epstopdf,epsfig}
\newtheorem{lemma}{Lemma}

\newtheorem{remark}{Remark}
\newtheorem{theorem}{Theorem}

\definecolor{deanPURPLE}{rgb}{.3,0,.5}

\newcommand{\li}{LING }
\newcommand{\X}{{\bf X}}
\newcommand{\Y}{{\bf Y}}

\newcommand{\D}{{\bf D}}

\newcommand{\U}{{\bf U}}
\newcommand{\V}{{\bf V}}
\newcommand{\specialcell}[2][c]{\begin{tabular}[#1]{@{}c@{}}#2\end{tabular}}

\title{Fast Ridge Regression with Randomized Principal Component Analysis and Gradient Descent}

\author{ {\bf Yichao Lu} \thanks{ \quad yichaolu@wharton.upenn.edu} \and {\bf Dean P. Foster} \thanks{\quad foster@wharton.upenn.edu}\\
 \\
Department of Statistics\\ Wharton,
University of Pennsylvania\\ Philadelphia, PA, 19104-6340
}

\begin{document}




\maketitle
\begin{abstract} 
We propose a new two stage algorithm LING for large scale regression problems. LING has the same risk as the well known Ridge Regression under the fixed design setting and can be computed much faster. Our experiments have shown that LING performs well in terms of both prediction accuracy and computational efficiency compared with other large scale regression algorithms like Gradient Descent, Stochastic Gradient Descent and Principal Component Regression on both simulated and real datasets.
\end{abstract} 

\section{Introduction}
Ridge Regression (RR) is one of the most widely applied penalized regression algorithms in machine learning problems. Suppose $\X$ is the $n\times p$ design matrix and $\Y$ is the $n\times 1$ response vector, ridge regression tries to solve the problem
\begin{equation}
\hat{\beta}=\arg\min\|\X\hat{\beta}-\Y\|^2+n\lambda\|\hat{\beta}\|^2
\label{target}
\end{equation}
which has an explicit solution 
\begin{equation}
\hat{\beta}=(\X^{\top}\X+n\lambda)^{-1}\X^{\top}\Y
\label{explicit}
\end{equation} 
However,  for modern problems with huge design matrix $\X$, computing (\ref{explicit}) costs $O(np^2)$ FLOPS. When $p > n\gg 1$ one can consider the dual formulation of (\ref{target}) which also has an explicit solution as mentioned in \citep{yichao13,saunders98} and the cost is $O(n^2p)$ FLOPS. In summary, trying to solve (\ref{target}) exactly costs $O(np \min{\{n,p\}})$ FLOPS which can be very slow.\\
There are faster ways to approximate (\ref{explicit}) when computational cost is the concern. One can view RR as an optimization problem and use Gradient Descent (GD) which takes $O(np)$ FLOPS in every iteration. However, the convergence speed for GD depends on the spectral of $\X$ and $\lambda$. When $\X$ is ill conditioned, GD requires a huge number of iterations to converge which makes it very slow. For huge datasets, one can also apply stochastic gradient descent (SGD) \citep{zhang04,johnson13,bottou2010}, a powerful tool for solving large scale optimization problems. \\
Another alternative for regression on huge datasets is Principle Component Regression (PCR) as mentioned in \citep{artemiou09,Ian05}, which runs regression only on the top $k_1$ principal components of the $\X$ matrix. PCA for huge $\X$ can be computed efficiently by randomized algorithms like \citep{tropprandom,halko11}. The cost for computing top $k_1$ PCs of $X$ is $O(npk_1)$ FLOPS. The connection between RR and PCR is well studied by \citep{dhillonrisk}. The problem of PCR is that when a large proportion of signal sits on the bottom PCs, it has to regress on a lot of PCs which makes it both slow and inaccurate.\\
In this paper, we propose a two stage algorithm LING\footnote{LING is the Chinese of ridge} which is a faster way of computing the RR solution (\ref{explicit}). A detailed description of the algorithm is given in section 2. In section 3, we prove that LING has the same risk as RR under the fixed design setting. In section 4, we compare the performance of PCR, GD, SGD and LING in terms of prediction accuracy and computational efficiency on both simulated and real data sets.

\section{The Algorithm}
\subsection{Description of the Algorithm}
\li is a two stage algorithm. The intuition of \li is quite straight forward. Note that regressing $\Y$ on $\X$ is essentially projecting $\Y$ onto the span of $\X$. Let $\U_1$ denotes the top $k_2$ PCs (singular vectors) of $\X$ and let $\U_2$ denote the remaining PCs. The projection of $\Y$ onto the span of $\X$ can be decomposed into two orthogonal parts, the projection onto $\U_1$ and the projection onto $\U_2$. In the first stage, we pick a $k_2\ll p$ and the projection onto $\U_1$ can be computed directly by $\hat{\Y}_1=\U_1\U_1^{\top}\Y$ which is exactly the same as running a PCR on top $k_2$ PCs. For huge $\X$, computing the top $k_2$ PCs exactly is very slow, so we use a faster randomized SVD algorithm for computing $\U_1$ which is proposed by \citep{tropprandom} and described below. In the second stage, we first compute $\Y_r=\Y-\hat{\Y}_1$ and $\X_r=\X-\U_1\U_1^{\top}\X$ which are the residual of $\Y$ and $\X$ after projecting on $\U_1$. Then we compute the projection of $\Y$ onto the span of $\U_2$ by solving the optimization problem $\min_{\hat{\gamma}_2 \in \mathcal{R}^p}\|\X_r\hat{\gamma}_2-\Y_r\|^2$ with GD (Algorithm \ref{al3}). Finally, since RR shrinks the projection of $\Y$ onto $\X$ via regularization, we also shrink the projections in both stages accordingly. Shrinkage in the first stage is performed directly on the estimated regression coefficients and shrinkage in the second stage is performed by adding a regularization term to the optimization problem mentioned above.
Detailed description of \li is shown in Algorithm \ref{al1}.

\begin{algorithm}[tb]
\caption{LING}
\begin{algorithmic}\label{algo:1}
\STATE
{\bf Input :}  Data matrix $\X$ ,$\Y$. $\U_1$, an orthonormal matrix consists of top $k_2$ PCs of $\X$. $d_1,d_2,...d_{k_2}$, top $k_2$ singular values of $\X$. Regularization parameter $\lambda$, an initial vector $\hat{\gamma}_{2,0}$ and number of iterations $n_2$ for GD .
\STATE
{\bf Output :}  $\hat{\gamma}_{1,s}$, $\hat{\gamma}_2$, the regression coefficients.

\STATE 1.Regress $\Y$ on $\U_1$, let $\hat{\gamma}_1=\U_1^{\top}\Y$.

\STATE 2.Compute the residual of previous regression problem.
Let $\Y_{r}=\Y-\U_1\hat{\gamma}_1$.

\STATE 3.Compute the residual of $\X$ regressing on $\U_1$. Use $\X_r=\X-\U_1\U_1^{\top}\X$ to denote the residual of $\X$.

\STATE 4.Use gradient descent with optimal step size with initial value $\hat{\gamma}_{2,0}$ (see algorithm 3)
to solve the RR problem $\min_{\hat{\gamma}_2 \in \mathcal{R}^p}\|\X_r\hat{\gamma}_2-\Y_r\|^2+n\lambda\|\hat{\gamma}_2\|^2$.

\STATE 5. Compute a shrinkage version of $\hat{\gamma}_1$ by $(\hat{\gamma}_{1,s})_i=\frac{d_i^2}{d_i^2+n\lambda}(\hat{\gamma}_1)_i$

\STATE 6.The final estimator is $\hat{\Y}=\U_1\hat{\gamma}_{1,s}+ \X_r\hat{\gamma}_2$.\\

\end{algorithmic}
\label{al1}
\end{algorithm}

\begin{algorithm}[tb]
\caption{Random SVD}
\begin{algorithmic}

\STATE {\bf Input :}  design matrix $\X$, target dimension $k_2$, number of power iterations $i$.
\STATE {\bf Output :} $\U_1\in n\times k_2$, the matrix of top $k_2$ left singular vectors of $\X$, $d_1,d_2,...d_{k_2}$, the top $k_2$ singular values of $\X$.

\STATE 1.Generate random matrix $R_1\in p\times k_2$ with i.i.d standard Gaussian entries.

\STATE 2.Estimate the span of top $k_2$ left singular vectors of $\X$ by $A_1=(\X\X^{\top})^i\X R_1$.

\STATE 3.Use QR decomposition to compute $Q_1$ which is an orthonormal basis of the column space of $A_1$.

\STATE 4.Compute SVD of the reduced matrix $Q_1^{\top}\X=\U_0\D_0\V_0^{\top}$.

\STATE 5.$\U_1=Q_1\U_0$ gives the top $k_2$ singular vectors of $\X$ and the diagonal elements of $\D_0$ gives the singular values.

\end{algorithmic}
\label{al2}
\end{algorithm}

\begin{algorithm}

\caption{Gradient Descent with Optimal Step Size (GD)}
\begin{algorithmic}

\STATE {\bf Goal :} Solve the ridge problem $\min_{\hat{\gamma} \in \mathcal{R}^p}\|\X\hat{\gamma}-\Y\|^2+n\lambda\|\hat{\gamma}\|^2$.

\STATE {\bf Input :} Data matrix $\X$, $\Y$, regularization parameter $\lambda$, number of iterations $n_2$, an initial vector $\hat{\gamma}_{0}$
 \STATE {\bf Output :} $\hat{\gamma}$

 \FOR{$t=0$ {\bfseries to} $n_2-1$}
 \STATE$Q=2\X^{\top}\X+2n\lambda I$
 \STATE$w_t=2\X^{\top}\Y-Q\hat{\gamma}_{t}$
 \STATE$s_t=\frac{w_t^{\top}w_t}{w_t^{\top}Qw_t}$. $s_t$ is the step size which makes the target function decrease the most in direction $w_t$.
 \STATE $\hat{\gamma}_{t+1}=\hat{\gamma}_{t}+s_t\cdot w_t$.
 \ENDFOR

\end{algorithmic}
\label{al3}
\end{algorithm}
\begin{remark}
\li can be regarded as a combination of PCR and GD. The first stage of LING is a very crude estimation of $\Y$ and the second stage adds a correction to the first stage estimator. Since we don't need a very accurate estimator in the first stage it suffices to pick a very small $k_2$ in contrast with the $k_1$ PCs needed for PCR. In the second stage, the design matrix $\X_r$ is a much better conditioned matrix than the original $\X$ since the directions with largest singular value have been removed. As introduced in section 2.2, Algorithm  \ref{al3} converges much faster with a better conditioned matrix. Hence GD in the second stage of \li converges faster than directly applying GD for solving equation \ref{target}. The above property guarantees that LING is both fast and accurate compared with PCR and GD. More details about on the computational cost will be discussed in section 2.2.
\label{r1}
\end{remark}
\begin{remark}
Algorithm \ref{al2} is essentially approximating the subspace of top left singular vectors by random projection. It provides a fast approximation of the top singular values and vectors for large $\X$ when computing the exact SVD is very slow.  Theoretical guarantees and more detailed explanations can be found in \citep{tropprandom}. Empirically we find in the experiments, Algorithm \ref{al2} may occasionally generate a bad subspace estimator due to randomness which makes PCR perform badly. On the other hand, LING is much more robust since in the second stage it compensate for the signal missing in the first stage. In all the experiments, we set $i=1$.
\end{remark}


The shrinkage step (step 5) in Algorithm \ref{al1} is only necessary for theoretical purposes since the goal is to approximate Ridge Regression which shrinks the Least Square estimator over all directions. 
In practice shrinkage over the top $k_2$ PCs is not necessary. Usually the number of PCs selected ($k_2$) is very small. From the bias variance trade off perspective, the variance reduction gained from this shrinkage step is at most $O(\frac{k_2}{n})$ under the fixed design setting \citep{dhillonrisk} which is a tiny number. Moreover, since the top singular values of $\X^{\top}\X$ are usually very large compared with $n\lambda$ (since large $\lambda$ will introduce large bias), the shrinkage factor $\frac{d_i^2}{d_i^2+n\lambda}$ will be pretty close to $1$ for top singular values. We use shrinkage in  Algorithm \ref{al1} because the risk of the shrinkage version of LING is exactly the same as RR as proved in section 3.\\
Algorithm \ref{al2} can be further simplified if we skip the shrinkage step mentioned in previous paragraph. Instead of computing the top $k_2$ PCs, the only thing we need to know is the subspace spanned by these PCs since the first stage is essentially projecting $\Y$ onto this subspace. In other words, we can replace $\U_1$ in step 1, 2, 3 of Algorithm \ref{al1} with $Q_1$ obtained in step 3 of  Algorithm \ref{al2} and directly let $\hat{\Y}=Q_1\hat{\gamma}_{1}+ \X_r\hat{\gamma}_2$. In the experiments of section 4 we use this simplified version.\\

\subsection{Computational Cost}
We claim that the cost of LING is $O\big(np(k_2+n_2)\big)$ where $k_2$ is the number of PCs used in the first stage and $n_2$ is the number of iterations of GD in the second stage. According to \citep{tropprandom}, the dominating step in Algorithm \ref{al2} is computing $(\X\X^{\top})^i\X R_1$ and $Q_1^{\top}\X$ which costs $O(npk_2)$ FLOPS. Computing $\hat{\gamma}_1$ and $\Y_r$ cost less than $O(npk_2)$. Computing $\X_r$ costs $O(npk_2)$. So the computational cost before the GD step is $O(npk_2)$. For the GD stage, note that in every iteration $Q$ never need to be constructed explicitly. While computing $w_t$ and $s_t$, always multiply matrix  and vector first gives a cost of $O(np)$ for every iteration. So the cost for GD stage is $O(n_2np)$. Add all pieces together the cost of LING is $O\big(np(k_2+n_2)\big)$ FLOPS.\\
Let $n_1$ be the number of iterations required for solving
(\ref{target}) directly by GD and $k_1$ be the number of PCs used for
PCR. Easy to check that the cost for GD is $O(n_1np)$ FLOPS and the
cost for PCR is $O(npk_1)$. As mentioned in remark \ref{r1}, the
advantage of LING over GD and PCR is that $k_1$ and $n_1$ might have
to be really large to achieve high accuracy but much smaller values of
the pair $(k_2$, $n_2)$ will work for LING. \\
Consider the case when the signal is widely spread among all PCs (the projection of $\Y$ onto the bottom PCs of $\X$ is large) instead of concentrating on the top ones, $k_1$ needs to be large to make PCR perform well since all the signals on bottom PCs are discarded by PCR. But LING doesn't need to include all the signals in the first stage regression since the signal left over will be estimated in the second GD stage. Therefore LING is able to recover most of the signal even with a small $k_2$. \\
In order to understand the connection between accuracy and number of iterations in Algorithm \ref{al3} , we state the following theorem in \cite{marina07}:
\begin{theorem}
Let $g(z)=\frac{1}{2}z^{\top}Mz+q^{\top}z$ be a quadratic function where $M$ is a PSD matrix. Suppose $g(z)$ achieves minimum at $z^*$. Apply Algorithm \ref{al3} to solve the minimization problem. Let $z_t$ be the $z$ value after $t$ iterations, then the gap between $g(z_t)$ and $g(z^*)$, the minimum of the objective function satisfies
\begin{equation}
\frac{g(z_{t+1})-g(z^*)}{g(z_t)-g(z^*)}\le C=\left(\frac{A-a}{A+a}\right)^2
\end{equation}
where $A,a$ are the largest and smallest eigenvalue of the $M$ matrix.
\label{converge}
\end{theorem}
Theorem \ref{converge} shows that the sub optimality of the target function decays exponentially as the number of iterations increases and the speed  of decay depends on the largest and smallest singular value of the PSD matrix that defines the quadratic objective function.
If we directly apply GD to solve (\ref{target}), Let $f_1(\beta)=\|\X\beta-\Y\|^2+n\lambda\|\beta\|^2$. Assume $f_1$ reaches its minimal at $\hat{\beta}$. Let $\hat{\beta}_{t}$ be the coefficient after $t$ iterations and let $d_i$ denote the $i^{th}$ singular value of $\X$. Apply theorem \ref{converge} we have
\begin{equation}
\frac{f_1(\hat{\beta}_{t+1})-f_1(\hat{\beta})}{f_1(\hat{\beta}_{t})-f_1(\hat{\beta})}\le C =\left(\frac{d_{1}^2-d_{p}^2}{d_{1}^2+d_{p}^2+2n\lambda}\right)^2
\label{rgd}
\end{equation}
Similarly for the second stage of LING,
Let $f_2(\gamma_2)=\|\X_r\gamma_2-\Y_r\|^2+n\lambda\|\gamma_2\|^2$. Assume $f_2$ reaches its minimal at $\hat{\gamma}_2$. We have
\begin{equation}
\frac{f_2(\hat{\gamma}_{2,t+1})-f_2(\hat{\gamma}_2)}{f_2(\hat{\gamma}_{2,t})-f_2(\hat{\gamma}_2)}\le C=\left(\frac{d_{k_2+1}^2}{d_{k_2+1}^2+2n\lambda}\right)^2
\label{rpcgd}
\end{equation}
In most real problems, the top few singular values of $\X^{\top}\X$ are much larger than the other singular values and $n\lambda$. Therefore the constant $C$ obtained in (\ref{rgd}) can be very close to 1 which makes GD algorithm converges very slowly. On the other hand, removing the top few PCs will make $C$ in (\ref{rpcgd}) significantly smaller than 1. In other words, GD may take a lot of iterations to converge when solving (\ref{target}) directly while the second stage of LING takes much less iterations to converge. This can also be seen in the experiments of section 4.
\section{Theorems}
In this section we compute the risk of LING estimator under the fixed design setting. For simplicity, assume $\U_1,\D_0$ generated by Algorithm \ref{al2} give exactly the top $k_2$ left singular vectors and singular values of $\X$ and GD in step 4 of Algorithm  \ref{al1} converges to the optimal solution. Let $\X=\U\D\V^\top$ be the SVD of $\X$ where $\U=(\U_1,\U_2)$ and $\V=(\V_1,\V_2)$. Here $\U_1,\V_1$ are top $k_2$ singular vectors and $\U_2, \V_2$ are bottom $p-k_2$ singular vectors. Let $\D=\text{diag}(\D_1,\D_2)$ where $\D_1\in k_2\times k_2$ contains top $k_2$ singular values denoted by $d_1\ge d_2\ge...\ge d_{k_2}$ and $\D_2\in p-k_2\times p-k_2$ contains bottom $p-k_2$ singular values. Let $\D_3=\text{diag}({\bf 0},\D_2)$ (replace $\D_1$ in $\D$ by a zero matrix of the same size).\\
\subsection{The Fixed Design Model}
Assume $\X$, $\Y$ comes from the fixed design model $\Y=\X\beta+\epsilon$ where $\epsilon\in n\times 1$ are i.i.d Gaussian noise with variance $\sigma^2$. Note that $\X=\U_1\D_1\V_1^{\top}+\X_r$, the fixed design model can also be written as 
\[  \Y=(\U_1\D_1\V_1^{\top}+\X_r)\beta+\epsilon=\U_1\gamma_1+\X_r\gamma_2+\epsilon\] where $\gamma_1=\D_1\V_1^{\top}\beta$ and $\gamma_2=\beta$. We use the $l_2$ distance between $\mathbb{E}(\Y|\X)$ (the best possible prediction given $\X$) and $\hat{\Y}=\U_1\hat{\gamma}_{1,s}+\X_r\hat{\gamma}_2$ (the prediction by LING) as the loss function. The risk of LING can be written as
\begin{eqnarray*}
&&\frac{1}{n}\mathbb{E}\|\mathbb{E}(\Y|\X)-\U_1\hat{\gamma}_{1,s}-\X_r\hat{\gamma}_2\|^2\\
=&&\frac{1}{n}\mathbb{E}\|\U_1\gamma_1+\X_r\gamma_2-\U_1\hat{\gamma}_{1,s}-\X_r\hat{\gamma}_2\|^2
\end{eqnarray*}
We can further decompose the risk into two terms:
\begin{equation}
\begin{split}
&\frac{1}{n}\mathbb{E}\|\U_1\gamma_1+\X_r\gamma_2-\U_1\hat{\gamma}_{1,s}-\X_r\hat{\gamma}_2\|^2
=\\ &\frac{1}{n}\mathbb{E}\|\U_1\gamma_1-\U_1\hat{\gamma}_{1,s}\|^2
+\frac{1}{n}\mathbb{E}\|\X_r\gamma_2-\X_r\hat{\gamma}_2\|^2 
\end{split}
\label{rdecomp}
\end{equation}
because $\U_1^{\top}\X_r=0$. Note that here the expectation is taken with respect to $\epsilon$.\\ 
Let's calculate the two terms in ($\ref{rdecomp}$) separately. For the first term we have:
\begin{lemma}
\begin{equation}
\frac{1}{n}\mathbb{E}\|\U_1\gamma_1-\U_1\hat{\gamma}_{1,s}\|^2=\frac{1}{n}\sum_{j=1}^{k_2}\frac{d_j^4\sigma^2+\gamma_{1,j}^2n^2\lambda^2}{(d_j^2+n\lambda)^2}
\label{lemma1}
\end{equation}
Here $\gamma_{1,j}$ is the $j^{th}$ element of $\gamma_1$.
\end{lemma}
\begin{proof}
Let $S \in k_2\times k_2$ be the diagonal matrix with $S_{j,j}=\frac{d_j^2}{d_j^2+n\lambda}$. So we have $\hat{\gamma}_{1,s}=S\U_1^{\top}\Y=S\gamma_1+S\U_1^{\top}\epsilon$, $\mathbb{E}(\hat{\gamma}_{1,s})=S\gamma_1$. 
\begin{eqnarray}
&&\frac{1}{n}\mathbb{E}\|\U_1\gamma_1-\U_1\hat{\gamma}_{1,s}\|^2\nonumber\\
&=&\frac{1}{n}\mathbb{E}\|\U_1\mathbb{E}(\hat{\gamma}_{1,s})-\U_1\hat{\gamma}_{1,s}\|^2\nonumber\\
&&+\frac{1}{n}\|\U_1\gamma_1-\U_1\mathbb{E}(\hat{\gamma}_{1,s})\|^2\nonumber\\
&=&\frac{1}{n}\mathbb{E}\|\U_1S\U_1^{\top}\epsilon\|^2+\frac{1}{n}\|\gamma_1-S\gamma_1\|^2\nonumber\\
&=&\frac{1}{n}\mathbb{E}\text{Tr}(\U_1S^2\U_1^{\top}\epsilon\epsilon^{\top})+\frac{1}{n}\|\gamma_1-S\gamma_1\|^2\nonumber\\
&=&\frac{1}{n}\mathbb{E}\text{Tr}(S^2)\sigma^2+\frac{1}{n}\|\gamma_1-S\gamma_1\|^2\nonumber\\
&=&\frac{1}{n}\sum_{j=1}^{k_2}\frac{d_j^4\sigma^2+\gamma_{1,j}^2n^2\lambda^2}{(d_j^2+n\lambda)^2}\nonumber
\end{eqnarray}
\end{proof}
Now consider the second term in (\ref{rdecomp}).\\
Note that \[\X_r=\U\D_3\V^{\top}\]
The residual $\Y_r$ after the first stage can be represented by 

\[\Y_r=\Y-\U_1\hat{\gamma_1}=(I-\U_1\U_1^{\top})\Y=\X_r\gamma_2+(I-\U_1\U_1^{\top})\epsilon\]
and the optimal coefficient obtained in the second GD stage is
\[\hat{\gamma}_2=(\X_r^{\top}\X_r+n\lambda I)^{-1}\X_r^{\top}\Y_r\]
For simplicity, let $\epsilon_2=(I-\U_1\U_1^{\top})\epsilon$. 
\begin{lemma}
\begin{equation}
\mathbb{E}\|\X_r\gamma_2-\X_r\hat{\gamma}_2\|^2=\sum_{i=k_2+1}^{p}\frac{1}{(d_i^2+n\lambda)^2}(d_i^4\sigma^2+n\lambda^2d_i^2\alpha_i^2)
\label{lemma2}
\end{equation}
where $\alpha_i$ is the $i^{th}$ element of $\alpha=\V^{\top}\gamma_2$
\end{lemma}

\begin{proof}
Frist define
\begin{eqnarray}
\X_{\lambda}&=&\X_r^{\top}\X_r+n\lambda I \nonumber\\
\D_{\lambda}&=&\D_3^2+n\lambda I \nonumber
\end{eqnarray}

\begin{eqnarray}
\mathbb{E}\|\X_r\gamma_2-\X_r\hat{\gamma}_2\|^2&=&\|\X_r\gamma_2-\X_r\mathbb{E}(\hat{\gamma}_2)\|^2\label{bias}\\
&+&\mathbb{E}\|\X_r\mathbb{E}(\hat{\gamma}_2)-\X_r\hat{\gamma}_2\|^2\label{var}
\end{eqnarray}
Consider (\ref{bias}) and (\ref{var}) separately.
\begin{eqnarray}
(\ref{bias})&=&\|\X_r\X_{\lambda}^{-1}\X_r^{\top}\X_r\gamma_2-\X_r\gamma_2\|^2\nonumber\\
&=&\|\U\D_3\D_{\lambda}^{-1}\D_3^2\V^{\top}\gamma_2-\U\D_3\V^{\top}\gamma_2\|^2\nonumber\\
&=&\|\D_3\D_{\lambda}^{-1}\D_3^2\alpha-\D_3\alpha\|^2\nonumber\\
&=&\sum_{i=k_2+1}^{p}\alpha_i^2d_i^2(\frac{n\lambda}{d_i^2+n\lambda})^2\nonumber
\end{eqnarray}

\begin{eqnarray}
&&(\ref{var})=\mathbb{E}_{\epsilon_2}\|\X_r\X_{\lambda}^{-1}\X_r^{\top}\epsilon_2\|^2\nonumber\\
&=&\mathbb{E}_{\epsilon_2}\text{Tr}\big(\X_r\X_{\lambda}^{-1}\X_r^{\top}\X_r\X_{\lambda}^{-1}\X_r^{\top}\epsilon_2\epsilon_2^{\top}\big)\nonumber\\
&=&\mathbb{E}_{\epsilon_2}\text{Tr}\big(\D_3\D_{\lambda}^{-1}\D_3^2\D_{\lambda}^{-1}\D_3\U^{\top}\epsilon_2\epsilon_2^{\top}\U\big)\nonumber\\
&=&\text{Tr}\big(\D_3\D_{\lambda}^{-1}\D_3^2\D_{\lambda}^{-1}\D_3\mathbb{E}_{\epsilon_2}[\U^{\top}\epsilon_2\epsilon_2^{\top}\U]\big)\nonumber
\end{eqnarray}
Note that  
\begin{equation*}
\mathbb{E}_{\epsilon_2}[\U^{\top}\epsilon_2\epsilon_2^{\top}\U]=\text{diag}(0,I_{p-k_2})\sigma^2
\end{equation*} 
($\text{diag}(0,I_{p-k_2})$replace the top $k_2\times k_2$ block of the identity matrix with 0),
\begin{equation}
(\ref{var})=\sum_{i=k_2+1}^{p}\frac{d_i^4}{(d_i^2+n\lambda)^2}\sigma^2
\end{equation}
Add the two terms together finishes the proof.
\end{proof}

Plug (\ref{lemma1}) (\ref{lemma2}) into (\ref{rdecomp}) we have
\begin{theorem}
The risk of LING algorithm under fixed design setting is
\begin{equation}
\frac{1}{n}\sum_{j=1}^{k_2}\frac{d_j^4\sigma^2+\gamma_{1,j}^2n^2\lambda^2}{(d_j^2+n\lambda)^2}+\frac{1}{n}\sum_{i=k_2+1}^{p}\frac{d_i^4\sigma^2+n^2\lambda^2d_i^2\alpha_i^2}{(d_i^2+n\lambda)^2}
\end{equation}
\end{theorem}
\begin{remark}
This risk is the same as the risk of ridge regression provided by Lemma 1 in \citep{dhillonrisk}. Actually,  LING gets exactly the same prediction as RR on the training dataset.  This is very intuitive since on the training set LING is essentially decomposing the RR solution into the first stage shrinkage PCR predictor on top $k_2$ PCs and the second stage GD predictor over the residual spaces as explained in section 2. 
\end{remark}

\section{Experiments}

In this section we compare the accuracy and computational cost (evaluated in terms of FLOPS) of 3 different algorithms for solving Ridge Regression: Gradient Descent with Optimal step size (GD), Stochastic Variance Reduction Gradient (SVRG) \citep{johnson13} and LING. Here SVRG is an improved version of stochastic gradient descent which achieves exponential  convergence with constant step size. We also consider Principle Component Regression (PCR) \citep{artemiou09,Ian05}  which is another common way for running large scale regression. Experiments are performed on 3 simulated models and 2 real datasets. In general, LING performs well on all 3 simulated datasets while GD, SVRG and PCR fails in some cases. For two real datasets, all algorithms give reasonable performance while SVRG and LING are the best. Moreover, both stages of LING only requires only a moderate amount of matrix multiplications each cost $O(np)$, much faster to run on matlab compared with SVRG which contains a lot of loops.
\subsection{Simulated Data}
Three different datasets are constructed based on the fixed design model $\Y=\X\beta+\epsilon$ where $\X$ is of size $2000\times 1500$. In the three experiments $\X$ and $\beta$ are generated randomly in different ways (more details in following sections) and i.i.d Gaussian noise is added to $\X\beta$ to get $\Y$. Then GD, SVRG, PCR and LING are performed on the dataset. For GD, we try different number of iterations $n_1$. For SVRG, we vary the number of passes through data denoted by $n_{\text{SVRG}}$. The number of iterations SVRG takes equals the number of passes through data times sample size and each iteration takes $O(p)$ FLOPS. The step size of SVRG is chosen by cross validation but this cost is not considered when evaluating the total computational cost. Note that one advantage of GD and \li is that due to the simple quadratic form of the target function, their step size can be computed directly from the data without cross validation which introduces extra cost. For PCR we pick different number of PCs ($k_1$). For LING we pick top $k_2$ PCs in the first stage and try different number of iterations $n_2$ in the second stage. The computational cost and the risk of the four algorithms are computed. The above procedure is repeated over 20 random generation of $\X$, $\beta$ and $\Y$. The risk and computational cost of the traditional RR solution (\ref{explicit})  for every dataset is also computed as a benchmark.\\The parameter set up for the three datasets are listed in table \ref{t1}.
\begin{table}[t]
\caption{parameter setup for simulated data}
\vskip 0.15in
\begin{center}
\begin{small}
\begin{sc}

\begin{tabular}{c|c|c|c}
\hline
& Model 1 & Model 2 & Model 3 \\
\hline
$k_1$&\specialcell{21,22,23,26\\30,50,100}&\specialcell{20,30,50\\100,150,400}&\specialcell{20,30,50,100\\150,400}\\
\hline
$n_1$&\specialcell{10,20,30\\50,80,100\\150,200}&\specialcell{2,4,6,8,10\\15,20,30}&\specialcell{6,10,15,20\\30,50,80\\120,180,250}\\
\hline
$k_2$&20&20&20\\
\hline
$n_2$&\specialcell{1,2,3,5\\8,13,20}&\specialcell{2,4,6,8,10\\15,20,30}&\specialcell{2,4,6,8,10\\15,30}\\
\hline
$n_{\text{SVRG}}$ &\specialcell{30,50,80\\120,150}&\specialcell{5,10,20\\30,50}&\specialcell{5,10,15,25\\40,60,90}\\
\hline
\end{tabular}
\end{sc}
\end{small}
\end{center}
\vskip -0.1in
\label{t1}
\end{table}
\subsubsection{Model 1}
In this model the design matrix $\X$ has a steep spectrum. The top 30 singular values of $\X$ decay exponentially as $1.3^i$ where $i=40,39,38...,11$. The spectrum of $\X$ is shown in figure \ref{ss}. To generate $\X$, we fix the diagonal matrix $\D_e$ with the designed spectrum and construct $\X$ by $\X=\U_e\D_e\V_e^{\top}$ where $\U_e$, $V_e$ are two random orthonormal matrices. The elements of $\beta$ are sampled uniformly from interval $[-2.5,2.5]$. Under this set up, most of the energy of the $\X$ matrix lies in top PCs since the top singular values are much larger than the remaining ones so PCR works well. But as indicated by (\ref{rgd}), the convergence of GD is very slow.\\
The computational cost and average risk of the four algorithms and also the RR solution (\ref{explicit}) over 20 repeats are shown in figure \ref{f1}. As shown in figure \ref{f1} both PCR and LING work well by achieving risk close to RR at less computational cost. SVRG is worse than PCR and LING but much better than GD.
\begin{figure}[h]
\centering
\includegraphics[width=7cm]{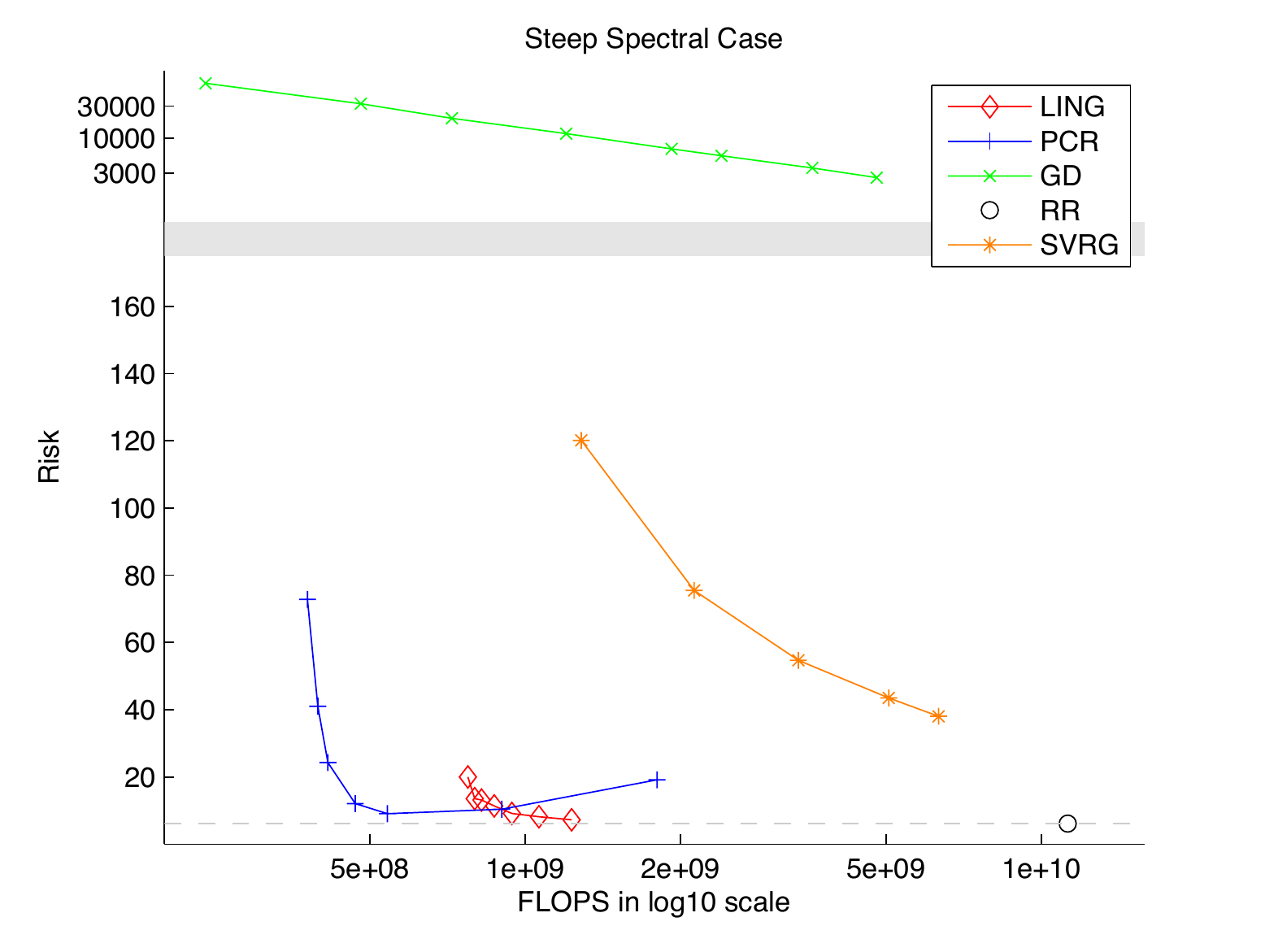}
\caption{{\bf Model 1:} Risk VS. Computational Cost plot. PCR and LING approaches the RR risk very fast. SVRG also approaches RR risk but cost more than the previous two. GD is very slow and inaccurate.}
\label{f1}
\end{figure}

\subsubsection{Model 2}
In this model the design matrix $\X$ has a flat spectrum. The singular values of $\X$ are sampled uniformly from $[\frac{\sqrt{2000}}{2},\sqrt{2000}]$. The spectrum of $\X$ is shown in figure \ref{fs}. To generate $\X$, we fix the diagonal matrix $\D_e$ with the designed spectrum and construct $\X$ by $\X=\U_e\D_e\V_e^{\top}$ where $\U_e$, $V_e$ are two random orthonormal matrices. The elements of $\beta$ are sampled uniformly from interval $[-2.5,2.5]$. Under this set up, the signal are widely spread among all PCs since the spectrum of $\X$ is relatively flat. PCR breaks down because it fails to catch the signal on bottom PCs. As indicated by (\ref{rgd}), GD converges relatively fast due to the flat spectrum of $\X$.\\
The computational cost and average risk of the four algorithms and also the RR solution (\ref{explicit}) over 20 repeats are shown in figure \ref{f2}. As shown by the figure GD works best since it approaches the risk of RR at the the lowest computational cost. LING  and SVRG also works by achieving reasonably low risk with less computational cost. PCR works poorly as explained before.
\begin{figure}[h]
\centering
\includegraphics[width=7cm]{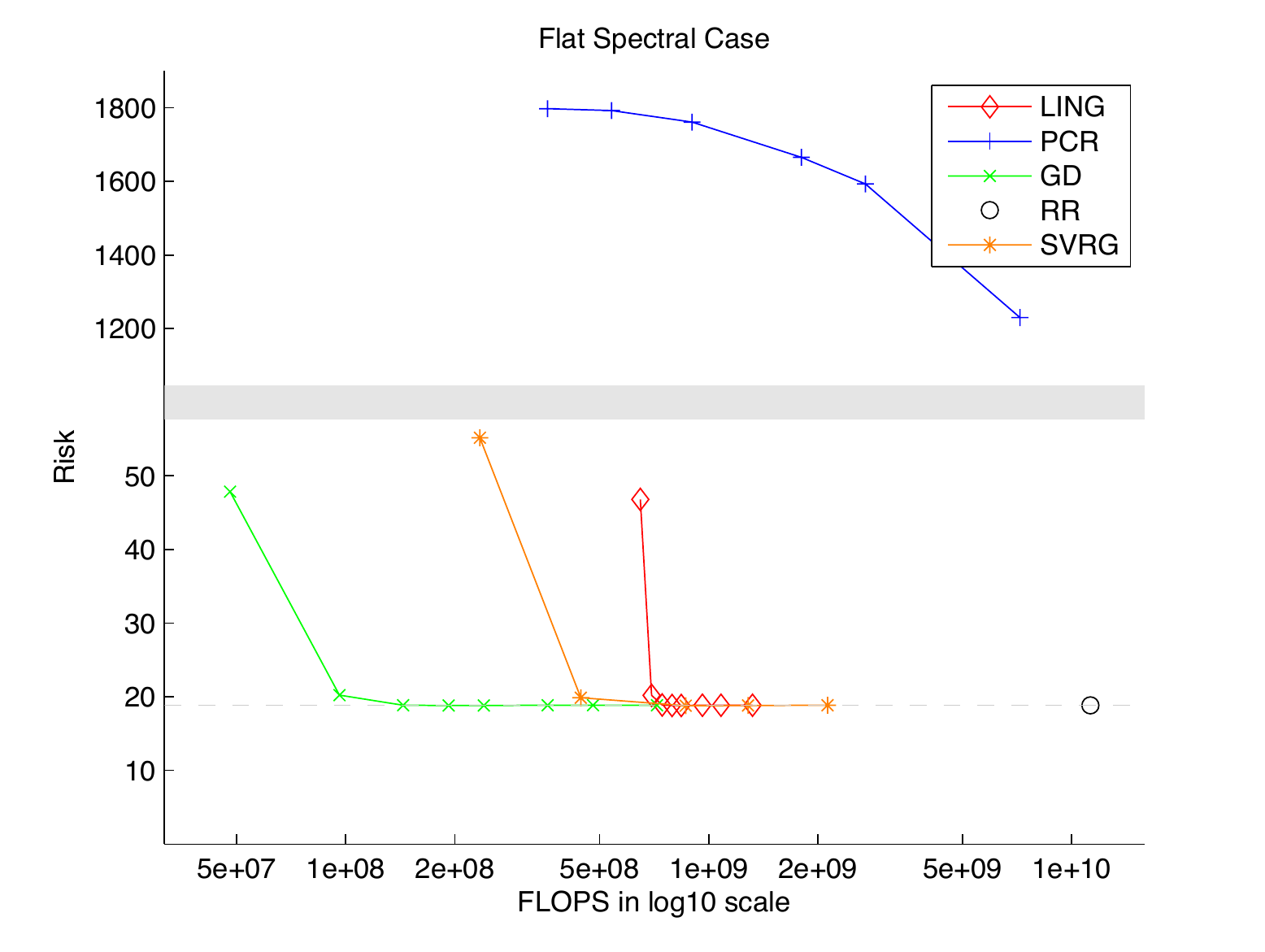}
\caption{{\bf Model 2:} Risk VS. Computational Cost plot. GD approaches the RR risk very fast. SVRG and LING are slower than GD but still achieves risk close to RR at less cost. PCR is slow and has huge risk.}
\label{f2}
\end{figure}
\subsubsection{Model 3}
This model presented a special case where both PCR and GD will break down. The singular values of $\X\in 2000\times 1500$ are constructed by first uniformly sample from $[\frac{\sqrt{2000}}{2},\sqrt{2000}]$. The top $15$ sampled values are then multiplied by $10$. The top $100$ sigular values of $\X$ are shown in figure \ref{es}. To generate $\X$, we fix the diagonal matrix $\D_e$ with the designed spectrum and construct $\X$ by $\X=\U_e\D_e$ where $\U_e$ is a random orthonormal matrix. The first $15$ and last $1000$ elements of the coefficient vector $\beta\in 1500 \times 1$ are sampled uniformly from interval $[-2.5,2.5]$ and other elements of $\beta$ remains $0$. In this set up, $\X$ has orthogonal columns which are the PCs, and the signal lies only on the top $15$ and bottom $1000$ PCs. PCR won't work since a large proportion of signal lies on the bottom PCs. On the other hand, GD won't work as well since the top few spectral values are too large compared with other singular values, which makes GD converges very slowly.\\
The computational cost and risk of the four algorithms and also the RR solution (\ref{explicit}) over 20 repeats are shown in figure \ref{f3}. As shown by the figure LING works best in this set up. SVRG is slightly worse than LING but still approaching RR with less cost. In this case, GD converges slowly and PCR is completely off target as explained before.
\begin{figure}[h]
\centering
\includegraphics[width=7cm]{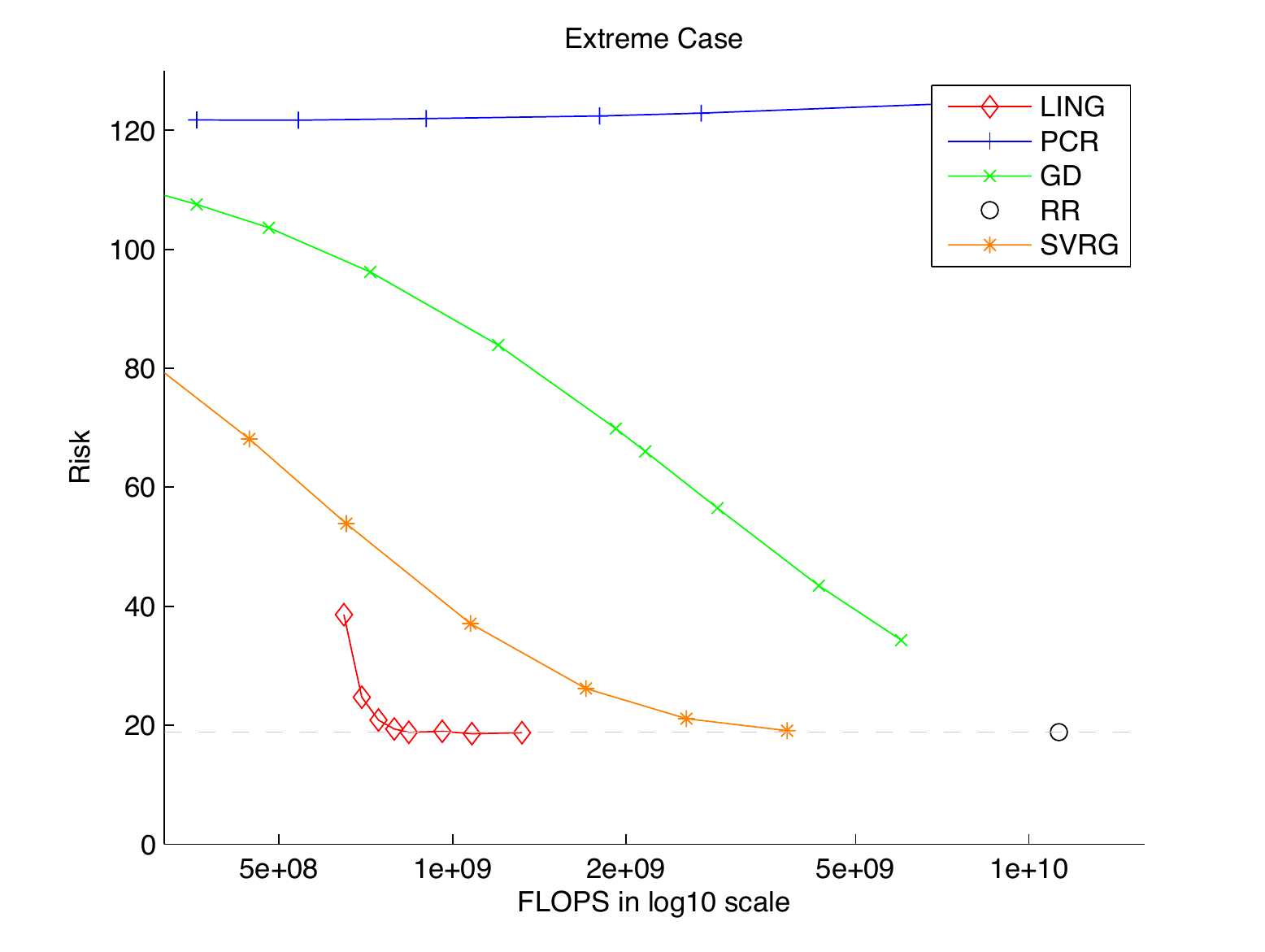}
\caption{{\bf Model 3:} Risk VS. Computational Cost plot. LING approaches RR risk the fastest. SVRG is slightly slower than LING. GD also approaches RR risk but cost more than LING. PCR has a huge risk no matter how many PCs are selected.}
\label{f3}
\end{figure}
\begin{figure}
\centering
\includegraphics[width=4.2cm]{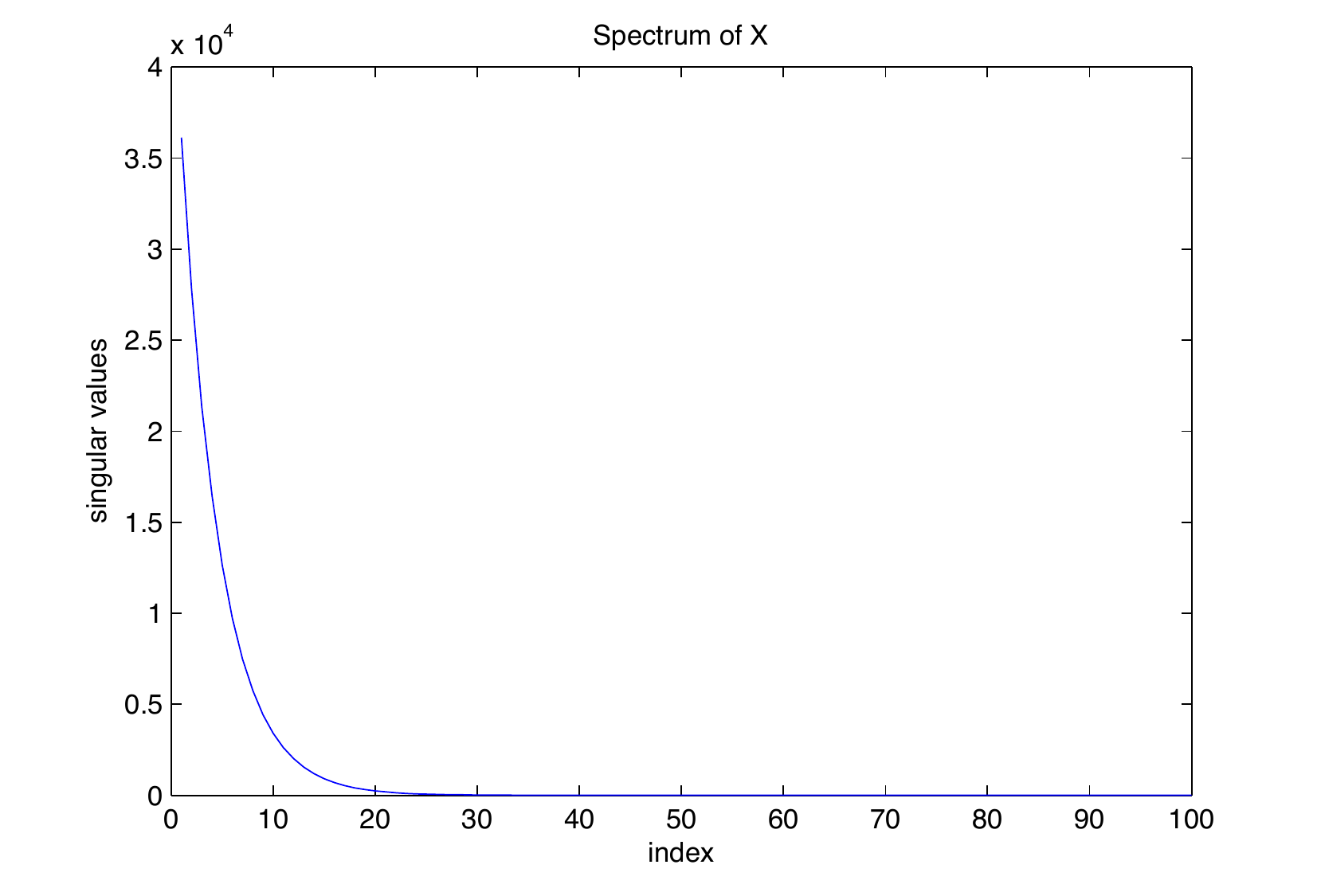}
\caption{Top $100$ singular values of $\X$ in Model 1} 
\label{ss}
\end{figure}
\begin{figure}[h] 
\centering 
\begin{minipage}[b]{0.45\linewidth}
\includegraphics[width=4.2cm]{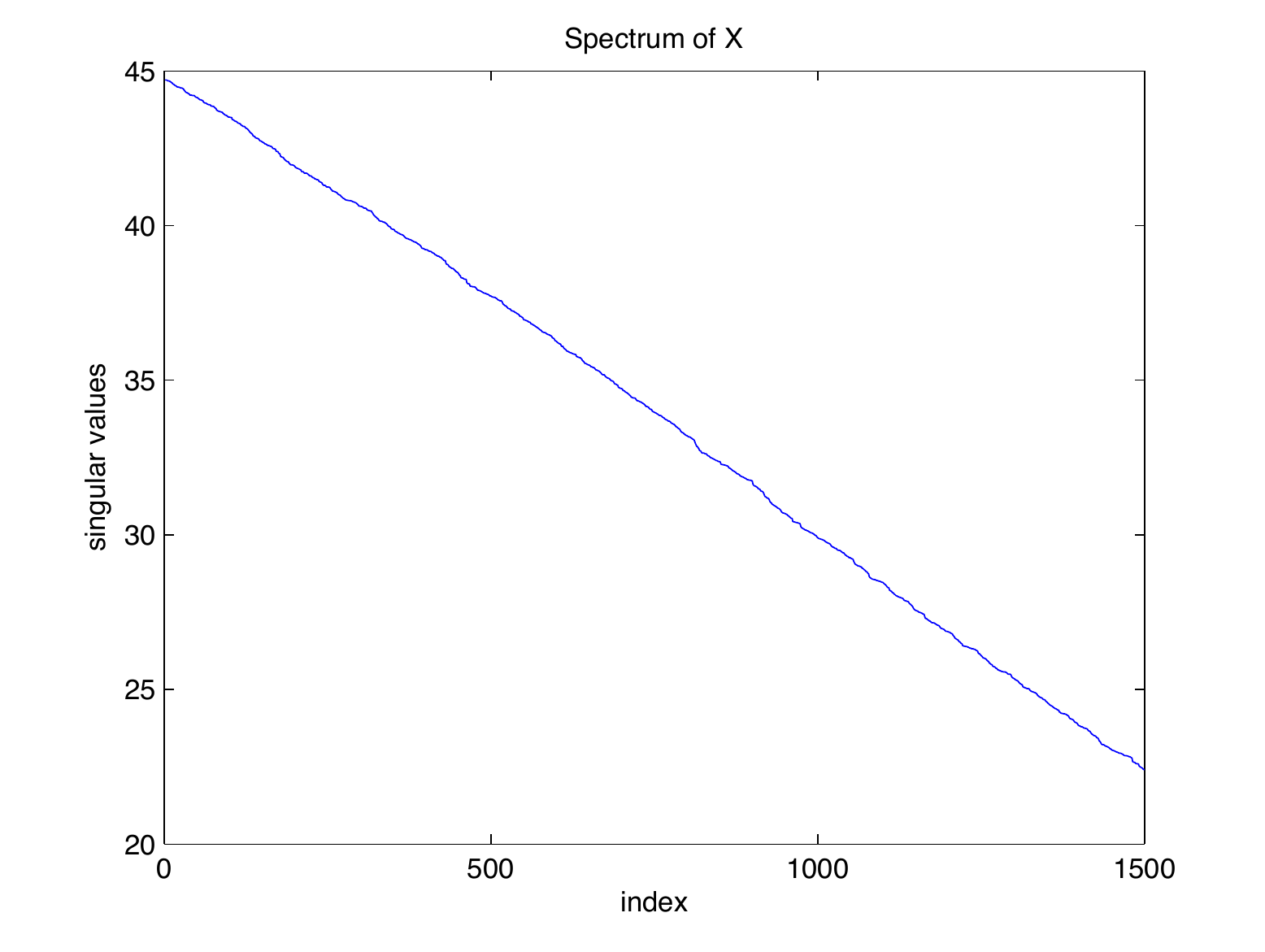}
\caption{Singular values of $\X$ in Model 2}
\label{fs}
\end{minipage} 
\quad
\begin{minipage}[b]{0.45\linewidth}
\includegraphics[width=4.2cm]{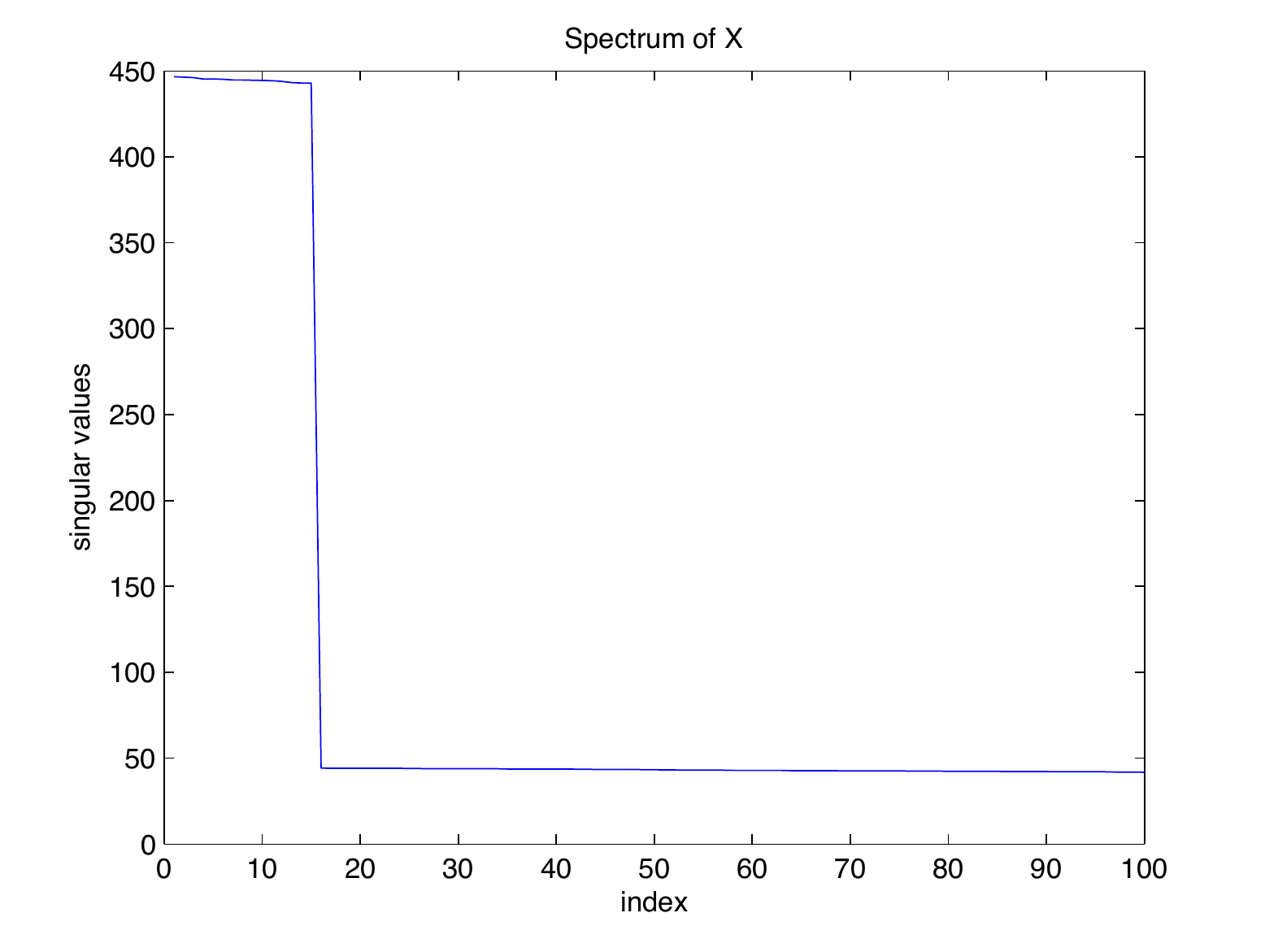}
\caption{Top $100$ singular values of $\X$ in Model 3} 
\label{es}
\end{minipage} 
 \end{figure}

\subsection{Real Data}
In this section we compare PCR, GD, SVRG and LING with the RR solution (\ref{explicit}) on two real datasets.\\
\subsubsection{Gisette Dataset}
The first is the gisette data set \citep{arcene} from the UCI repository which is a bi-class classification task. Every row of the design matrix $\X \in 6000\times 5000$ consists of  pixel features of a single digit "4" or "9" and $\Y$ gives the class label. Among the $6000$ samples, we use $5000$ for training and $1000$ for testing. The classification error rate for RR  solution (\ref{explicit}) is $0.019$. Since the goal is to compare different algorithms for regression,  we don't care about achieving the state of the art accuracy for this dataset as long as regression works reasonably well. When running PCR, we pick top $k_1=10,20,40,80,150,300,400$ PCs and in GD we iterate $n_1=2,5,10,15,20,30,50,100,150$ times. For SVRG we try $n_{\text{SVRG}}=1,2,3,5,10,20,40,80$ passes through the data. For LING we pick $k_2=5, 15$ PCs in the first stage and try $n_2=1,2,4,8,10,15,20,30,50$ iterations in the second stage. The computational cost and average classification error of the four algorithms and also the RR solution (\ref{explicit}) on test set over 6 different train test splits are shown in figure \ref{f4}. The top $150$ singular values of $\X$ are shown in figure \ref{gs}. As shown in the figure, SVRG gets close to the RR error very fast. The two curves of LING with $k_2=5,15$ are slower than SVRG since some initial FLOPS are spent on computing top PCs but after that they approach RR error very fast. GD also converges to RR but cost more than the previous two algorithms. PCR performs worst in terms of error and computational cost.
\begin{figure}[h]
\centering
\includegraphics[width=7cm]{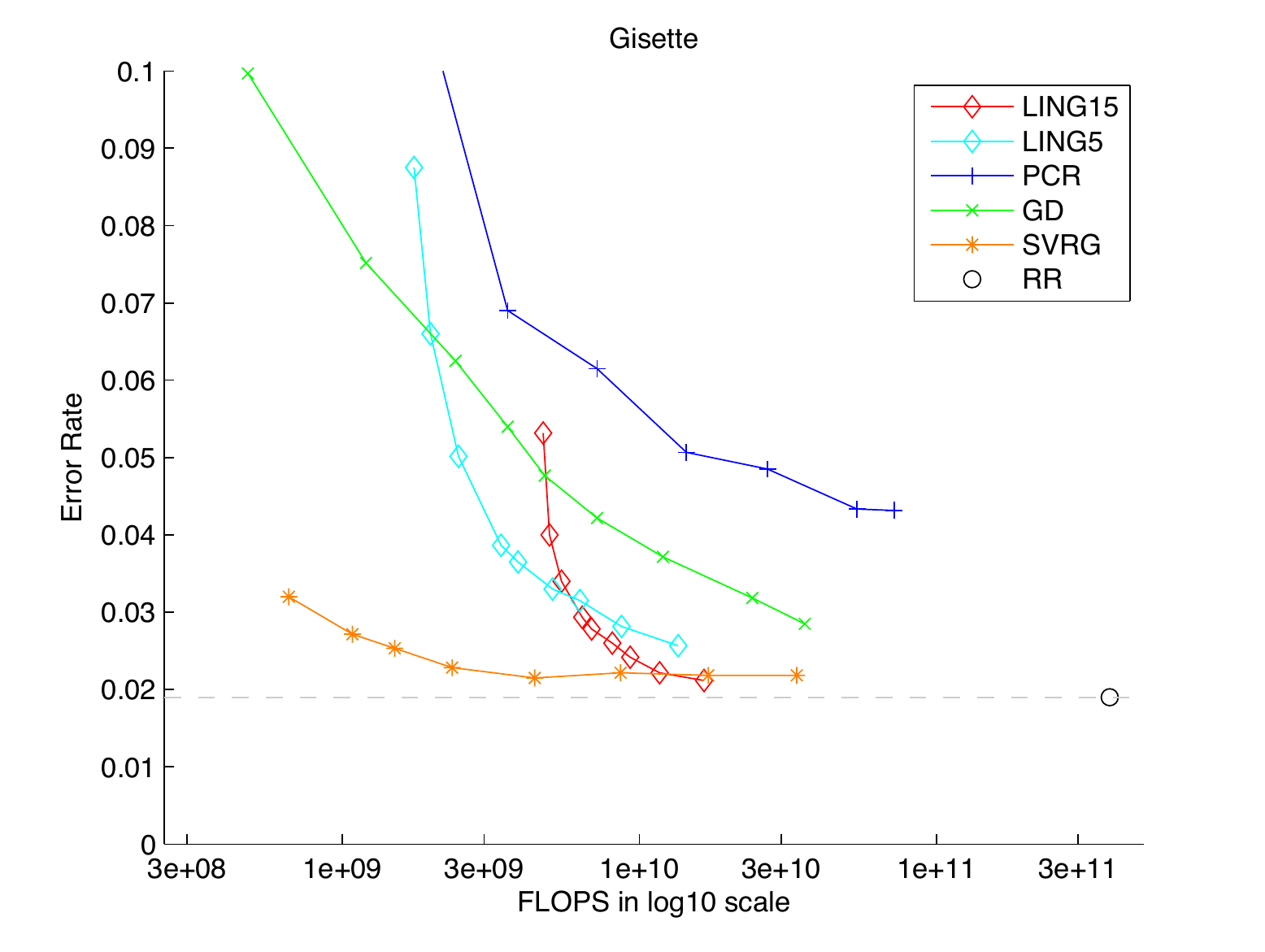}
\caption{{\bf Gisette:} Error Rate VS. Computational Cost plot. SVRG achieves  small error rate fastest. Two LING lines with different $n_2$ spent some FLOPS on computing top PCs first, but then converges very fast to a lower error rate. GD and PCR also provide reasonably small error rate and are faster than RR, but suboptimal compared with SVRG and LING.}
\label{f4}
\end{figure}
\subsubsection{Buzz in Social Media}
The second dataset is the UCI buzz in social media dataset which is a regression task. The goal is to predict popularity (evaluated by the mean number of active discussion) of a certain topic on Twitter over a period. The original feature matrix contains some statistics about this topic over that period like number of discussions created and new authors interacting at the topic. The original feature dimension is 77. We add quadratic interactions to make it 3080. To save time, we only used a subset of $8000$ samples. The samples are split into $6000$ train and $2000$ test. We use MSE on the test data set as the error measure. For PCR we pick $k_1=10,20,30,50,100,150$ PCs and in GD we iterate $n_1=1,2,4,6,8,10,15,20,30,40,60,100$ times. For SVRG we try $n_{\text{SVRG}}=1,2,3,5,10,15,20,40,80$ passes through the dataset and for LING we pick $k_2=5,15$ in the first stage and $n_2=0,1,2,4,6,8,10,15,20,25$ iterations in the second stage. The computational cost and average MSE on test set over 5 different train test splits are shown in figure \ref{f5}.  The top $150$ singular values of $\X$ are shown in figure \ref{bs}. As shown in the figure, SVRG approaches MSE of RR very fast. LING spent some initial FLOPS for computing top PCs but after that converges fast.  GD and PCR also achieves reasonable performance but suboptimal compared with SVRG and LING. The MSE of PCR first decays when we add more PCs into regression but finally goes up due to overfit.  

\begin{figure}[h]
\centering
\includegraphics[width=7cm]{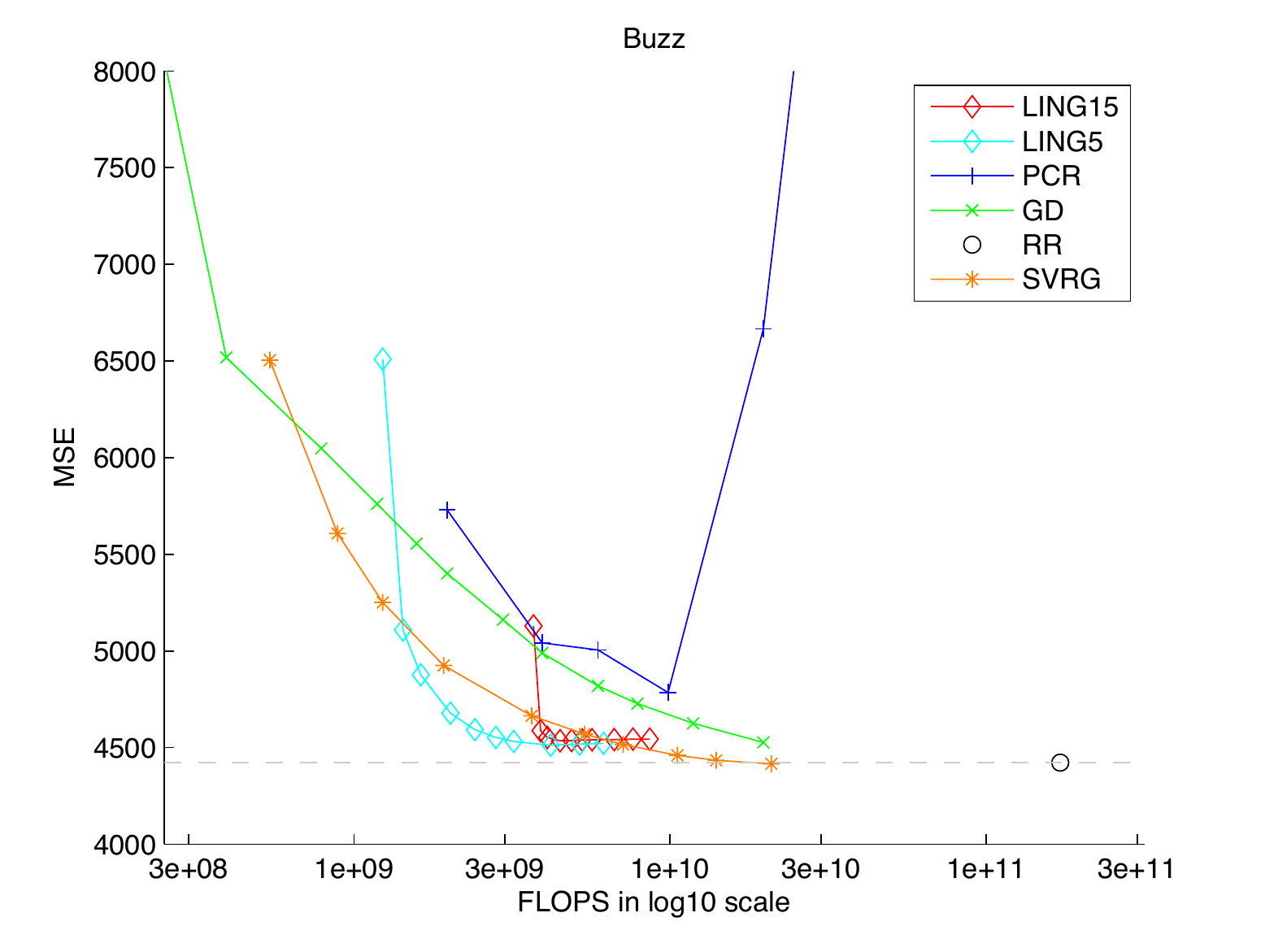}
\caption{{\bf Buzz :} MSE VS. Computational Cost plot. SVRG and two LING lines with different $n_2$ achieves small MSE fast. GD is slower than LING and SVRG. PCR reaches its smallest MSE at $k_1=50$ then overfits.}
\label{f5}
\end{figure}

\begin{figure}[h] 
\centering 
\begin{minipage}[b]{0.45\linewidth}
\includegraphics[width=4.2cm]{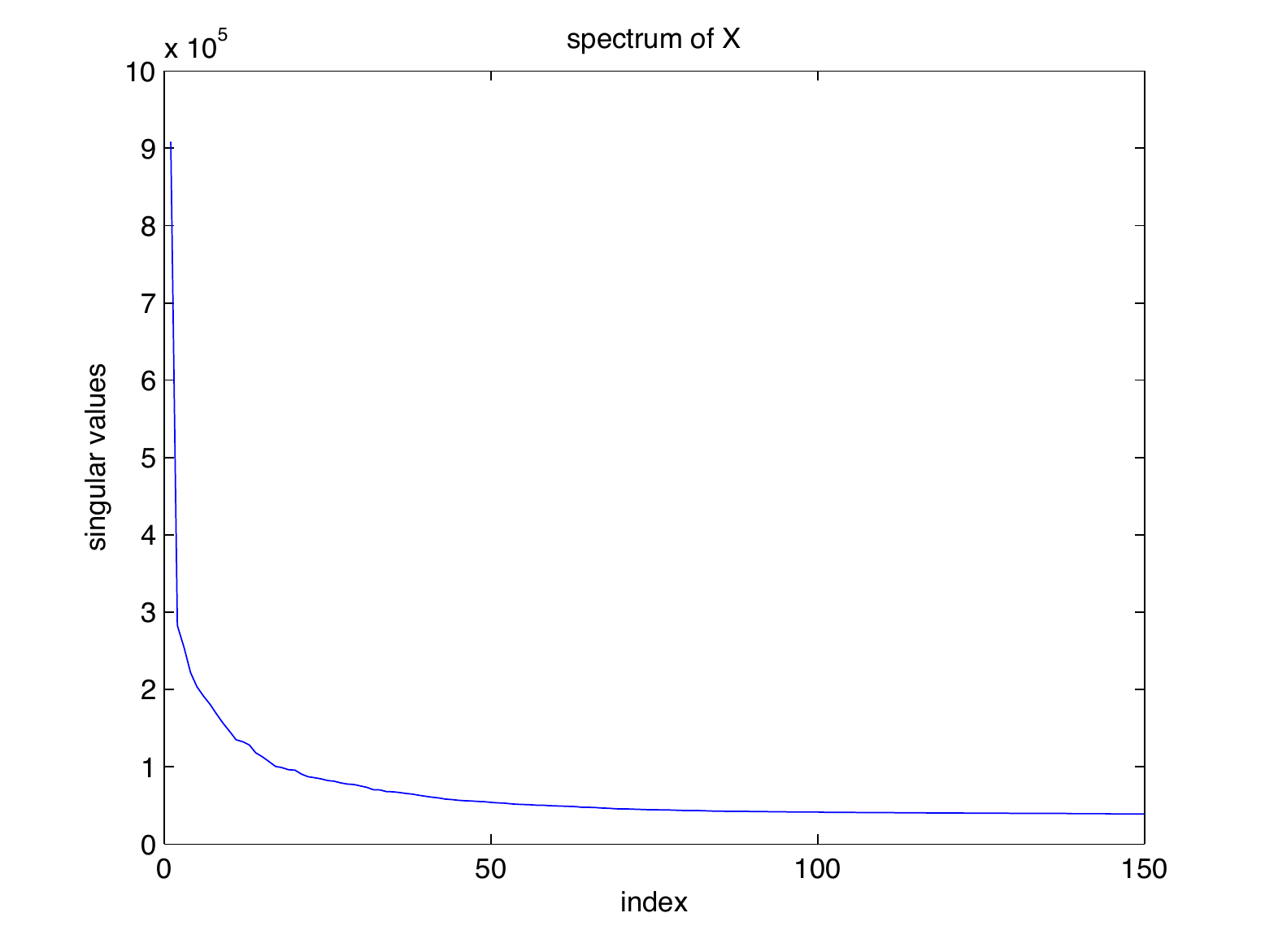}
\caption{Top 150 singular values of $\X$ in Gisette Dataset}
\label{gs}
\end{minipage} 
\quad
\begin{minipage}[b]{0.45\linewidth}
\includegraphics[width=4.2cm]{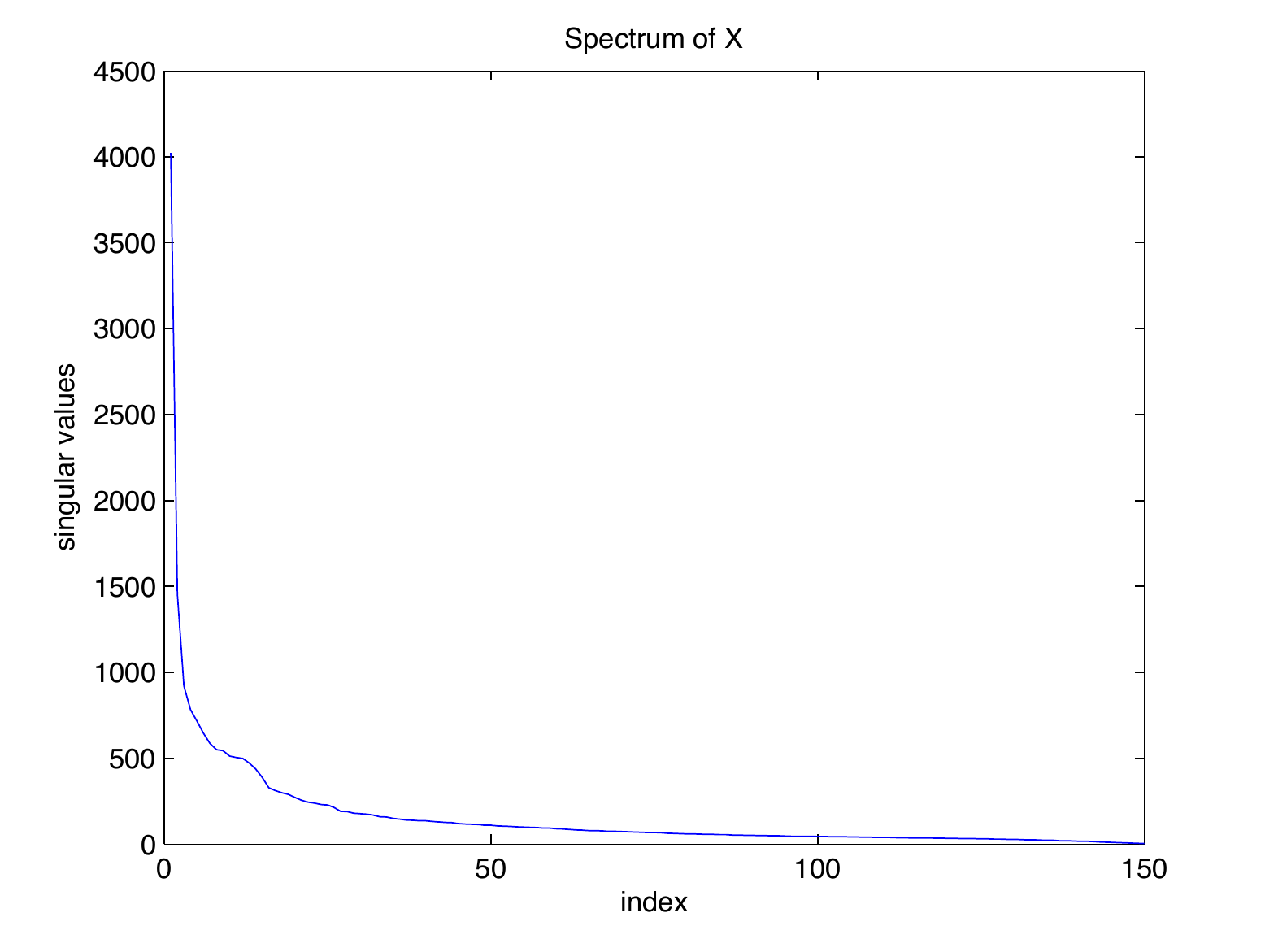}
\caption{Top $150$ singular values of $\X$ in Social Media Buzz Dataset} 
\label{bs}
\end{minipage} 
\end{figure}
\section{Summary}
In this paper we present a two stage algorithm LING for computing large scale Ridge Regression which is both fast and robust in contrast to the well known approaches GD and PCR. We show that under the fixed design setting LING actually has the same risk as Ridge Regression assuming convergence. In the experiments, LING achieves good performances on all datasets when compare with three other large scale regression algorithms.\\
We conjecture that same strategy can be also used to accelerate the convergence of stochastic gradient descent when solving Ridge Regression since the first stage in LING essentially removes the high variance directions of $\X$, which will lead to variance reduction for the random gradient direction generated by SGD.

\bibliography{pcr,pcgdr}

\begin{thebibliography}{12}
\providecommand{\natexlab}[1]{#1}
\providecommand{\url}[1]{\texttt{#1}}
\expandafter\ifx\csname urlstyle\endcsname\relax
  \providecommand{\doi}[1]{doi: #1}\else
  \providecommand{\doi}{doi: \begingroup \urlstyle{rm}\Url}\fi

\bibitem[A.Epelman(2007)]{marina07}
Marina A.Epelman.
\newblock Rate of convergence of steepest descent algorithm.
\newblock 2007.

\bibitem[Artemiou and Li(2009)]{artemiou09}
Andreas Artemiou and Bing Li.
\newblock On principal components and regression: a statistical explanation of
  a natural phenomenon.
\newblock \emph{Statistica Sinica}, 19\penalty0 (4):\penalty0 1557--1565, 2009.
\newblock ISSN 1017-0405.

\bibitem[Bottou(2010)]{bottou2010}
L\'{e}on Bottou.
\newblock {Large-Scale Machine Learning with Stochastic Gradient Descent}.
\newblock In Yves Lechevallier and Gilbert Saporta, editors, \emph{Proceedings
  of the 19th International Conference on Computational Statistics
  (COMPSTAT'2010)}, pages 177--187, Paris, France, August 2010. Springer.

\bibitem[Dhillon et~al.(2013)Dhillon, Foster, Kakade, and Ungar]{dhillonrisk}
Paramveer~S. Dhillon, Dean~P. Foster, Sham~M. Kakade, and Lyle~H. Ungar.
\newblock A risk comparison of ordinary least squares vs ridge regression.
\newblock \emph{Journal of Machine Learning Research}, 14:\penalty0 1505--1511,
  2013.

\bibitem[Guyon(2003)]{arcene}
Isabelle Guyon.
\newblock Design of experiments for the nips 2003 variable selection benchmark.
\newblock 2003.

\bibitem[Halko et~al.(2011{\natexlab{a}})Halko, Martinsson, and
  Tropp]{tropprandom}
N.~Halko, P.~G. Martinsson, and J.~A. Tropp.
\newblock Finding structure with randomness: Probabilistic algorithms for
  constructing approximate matrix decompositions.
\newblock \emph{SIAM Rev.}, 53\penalty0 (2):\penalty0 217--288, May
  2011{\natexlab{a}}.
\newblock ISSN 0036-1445.

\bibitem[Halko et~al.(2011{\natexlab{b}})Halko, Martinsson, Shkolnisky, and
  Tygert]{halko11}
Nathan Halko, Per-Gunnar Martinsson, Yoel Shkolnisky, and Mark Tygert.
\newblock An algorithm for the principal component analysis of large data sets.
\newblock \emph{SIAM J. Scientific Computing}, 33\penalty0 (5):\penalty0
  2580--2594, 2011{\natexlab{b}}.

\bibitem[Johnson and Zhang(2013)]{johnson13}
Rie Johnson and Tong Zhang.
\newblock Accelerating stochastic gradient descent using predictive variance
  reduction.
\newblock \emph{Advances in Neural Information Processing Systems (NIPS)},
  2013.

\bibitem[Jolliffe(2005)]{Ian05}
Ian Jolliffe.
\newblock \emph{Principal Component Analysis. Encyclopedia of Statistics in
  Behavioral Science}.
\newblock John Wiley \& Sons, 2005.

\bibitem[Lu et~al.(2013)Lu, Dhillon, Foster, and Ungar]{yichao13}
Yichao Lu, Paramveer~S. Dhillon, Dean Foster, and Lyle Ungar.
\newblock Faster ridge regression via the subsampled randomized hadamard
  transform.
\newblock In \emph{Advances in Neural Information Processing Systems (NIPS)},
  2013.

\bibitem[Saunders et~al.(1998)Saunders, Gammerman, and Vovk]{saunders98}
G.~Saunders, A.~Gammerman, and V.~Vovk.
\newblock {Ridge regression learning algorithm in dual variables}.
\newblock In \emph{Proc. 15th International Conf. on Machine Learning}, pages
  515--521. Morgan Kaufmann, San Francisco, CA, 1998.

\bibitem[Zhang(2004)]{zhang04}
Tong Zhang.
\newblock Solving large scale linear prediction problems using stochastic
  gradient descent algorithms.
\newblock In \emph{ICML 2004: Proceedings of the twenty-first International
  Conference on Machine Learning. OMNIPRESS}, pages 919--926, 2004.

\end{thebibliography}
\bibliographystyle{plainnat}

\end{document}